\documentclass[english]{article}
\usepackage[final]{nips_2017}

\usepackage[utf8]{inputenc} 
\usepackage[T1]{fontenc}    
\usepackage{hyperref}       
\usepackage{url}            
\usepackage{booktabs}       
\usepackage{amsfonts}       
\usepackage{nicefrac}       
\usepackage{microtype}      

\usepackage{varioref}
\usepackage{amsmath}
\usepackage{amsthm}
\usepackage{amssymb}
\PassOptionsToPackage{normalem}{ulem}
\usepackage{ulem}

\makeatletter

\theoremstyle{plain}
\newtheorem{thm}{Theorem}

\theoremstyle{plain}
\newtheorem{lem}[thm]{Lemma}

\makeatother

\begin{document}

\title{Computing the quality of the Laplace approximation}

\author{Guillaume Dehaene\\
Ecole Polytechnique Federale de Lausanne\\
\texttt{guillaume.dehaene@gmail.com}}
\maketitle
\begin{abstract}
Bayesian inference requires approximation methods to become computable,
but for most of them it is impossible to quantify how close the approximation
is to the true posterior. In this work, we present a theorem upper-bounding
the KL divergence between a log-concave target density $f\left(\boldsymbol{\theta}\right)$
and its Laplace approximation $g\left(\boldsymbol{\theta}\right)$.
The bound we present is computable: on the classical logistic regression
model, we find our bound to be almost exact as long as the dimensionality
of the parameter space is high.

The approach we followed in this work can be extended to other Gaussian
approximations, as we will do in an extended version of this work,
to be submitted to the Annals of Statistics. It will then become a
critical tool for characterizing whether, for a given problem, a given
Gaussian approximation is suitable, or whether a more precise alternative
method should be used instead.
\end{abstract}
Bayesian inference requires the following challenging computations:
given an unnormalized density $\tilde{f}\left(\boldsymbol{\theta}\right)=\exp\left(-\phi_{f}\left(\boldsymbol{\theta}\right)\right)$,
we must compute its integral $Z$ and then compute various expected
values under the normalized density $f=\tilde{f}/Z$. One possible
approximation for these computations consists in computing the Laplace
approximation of $\tilde{f}$ (\citet{bishop2006pattern}). This consists
of finding the maximum $\boldsymbol{\theta}^{\star}$ of $\tilde{f}\left(\boldsymbol{\theta}\right)$
and using the following Gaussian approximation:
\begin{equation}
\tilde{f}\left(\boldsymbol{\theta}\right)\approx\tilde{g}\left(\boldsymbol{\theta}\right)=\tilde{f}\left(0\right)\exp\bigg(-\underbrace{\frac{1}{2}\left(\boldsymbol{\theta}-\boldsymbol{\theta}^{\star}\right)H\phi_{f}\left(\boldsymbol{\theta}^{\star}\right)\left(\boldsymbol{\theta}-\boldsymbol{\theta}^{\star}\right)}_{=\phi_{g}\left(\boldsymbol{\theta}\right)}\bigg)\label{eq: the laplace approximation}
\end{equation}
where $H\phi_{f}$ is the \emph{Hessian }matrix of the second derivatives
of $\phi_{f}$.

There exists a theorem that justifies the use of this Laplace approximation:
the Bernstein-von Mises theorem (BvM; \citet{kleijn2012bernstein}).
This theorem is derived from a frequentist analysis: we treat the
data as random with some fixed probability distribution. The posterior
$f\left(\boldsymbol{\theta}\right)$ is then a function-valued random
variable which, under fairly general assumptions, becomes approximately
equal to the second function-valued random variable $g\left(\boldsymbol{\theta}\right)$
in the large-data limit.

This theorem mostly offers only a qualitative reassurance to the computationally
minded, as the assumptions are hard to check and involve inaccessible
terms. This is because current statements of BvM theorems are heavily
tailored towards showing that Bayesian inference is a valid frequentist
method and, namely, that it coincides with Maximum Likelihood Estimation
in the large-data limit. The object of this article is to instead
give a Bernstein-von Mises theorem that is aimed at characterizing
how good the Laplace approximation of one given $\tilde{f}\left(\boldsymbol{\theta}\right)$,
with no assumptions on how $\tilde{f}\left(\boldsymbol{\theta}\right)$
was generated, and while involving only quantities that are computable
from $\tilde{f}\left(\boldsymbol{\theta}\right)$.

\section{A deterministic Bernstein-von Mises theorem}

\paragraph{Assumptions.}

In order to derive our theorem, we must make some assumptions on the
target density $\tilde{f}\left(\boldsymbol{\theta}\right)$. In practice,
we use two which play very different roles:
\begin{enumerate}
\item A \emph{local} assumption constraining the higher-derivatives of $\phi_{f}\left(\boldsymbol{\theta}\right)$.\\
Indeed, the Laplace approximation is computed from a second-degree
Taylor expansion of $\phi_{f}\left(\boldsymbol{\theta}\right)$. In
order for this to be valid, the derivatives of $\phi_{f}\left(\boldsymbol{\theta}\right)$
need to be small.
\item A \emph{global} assumption, constraining the overall shape of $\tilde{f}\left(\boldsymbol{\theta}\right)$
to ensure that most of its mass resides in a close neighborhood of
$\boldsymbol{\theta}^{\star}$.\\
The object of this assumption is to keep us safe from trivial counter-examples
such as the following:
\begin{equation}
\tilde{f}\left(\theta\right)=\exp\left(-\frac{\theta^{2}}{2}\right)+\epsilon\exp\left(-\epsilon^{2}\frac{\theta^{2}}{2}\right)\label{eq: a trivial counter-example}
\end{equation}
If $\epsilon$ is small, the second term is virtually invisible and
the Laplace approximation appears to be good, but it is actually terrible
since the second term contributes one-half of the overall mass of
$\tilde{f}$.
\end{enumerate}
Some shape assumptions actually make stating a BvM theorem straightforward.
For example, if we are willing to assume that $\tilde{f}\left(\boldsymbol{\theta}\right)$
is $\beta$-strongly log-concave, i.e: $H\phi_{f}\left(\boldsymbol{\theta}\right)\geq\beta$,
we can then apply the log-Sobolev inequality (LSI; \citet{otto2000generalization}
Theorem 2) which upper-bounds the Kullback-Leibler divergence between
the normalized densities $g$ and $f$:
\begin{equation}
KL\left(g,f\right)\leq\frac{1}{2}E_{g}\left[\left(\nabla\phi_{f}\left(\boldsymbol{\theta}\right)-\nabla\phi_{g}\left(\boldsymbol{\theta}\right)\right)^{T}\beta^{-1}\left(\nabla\phi_{f}\left(\boldsymbol{\theta}\right)-\nabla\phi_{g}\left(\boldsymbol{\theta}\right)\right)\right]\label{eq: log-sobolev inequality}
\end{equation}
By performing a Taylor expansion of $\nabla\phi_{f}\left(\boldsymbol{\theta}\right)$
around $\boldsymbol{\theta}^{\star}$ and computing various moments
of $g$, we would obtain a complicated but computable upper-bound
for the distance between $g$ and $f$, dominated by the third derivative
of $\phi_{f}$. However, this inequality is almost as useless as it
is easy to derive because most models never lead to strongly log-concave
posterior distributions.

A more realistic assumption consists in assuming that $\tilde{f}\left(\boldsymbol{\theta}\right)$
is log-concave, i.e: $H\phi_{f}\left(\boldsymbol{\theta}\right)>0$
(see \citet{saumard2014log} for a review of the properties of log-concave
densities). This assumption holds for any model such that the log-prior
and the log-likelihood are both concave, e.g: logistic regression
with a Gaussian prior. \textbf{This is our global}\textbf{\emph{ }}\textbf{assumption}
which ensures that $\tilde{f}\left(\boldsymbol{\theta}\right)$ has
a single mode, and that its tails decay at least exponentially thus
guaranteeing that most of its mass is in a close neighborhood of $\boldsymbol{\theta}^{\star}$.

\paragraph{Strategy.}

We will only sketch our proof here and refer the interested reader
to our appendix for the detailed proof and the full expression of
our KL bound.

The key step is the following change of variable:
\begin{align*}
\boldsymbol{\theta} & \rightarrow\left(z,\boldsymbol{e}\right)\in\mathbb{R}*S^{d-1}\\
\boldsymbol{\theta} & =\boldsymbol{\theta}^{\star}+z^{2}\left(H\phi_{f}\left(\boldsymbol{\theta}^{\star}\right)\right)^{-1/2}\boldsymbol{e}
\end{align*}
where $S^{d-1}$ is the $d$-dimensional unit sphere. This change
of variable is such that the conditional density of the random variable
$z|\boldsymbol{e}$ is strongly log-concave, thus enabling us to apply
the LSI. We are then able to control the KL divergence between the
random variable $z|\boldsymbol{e}$ under the approximate model $g$
and the true model $f$. Denoting $\psi_{f,\boldsymbol{e}}\left(z\right)$
the negative log-density of $z|\boldsymbol{e}$ under $f$, we have:
\begin{equation}
KL\left(z_{g}|\boldsymbol{e},z_{f}|\boldsymbol{e}\right)\leq\frac{E_{g}\left[\left(\psi_{f,\boldsymbol{e}}^{'}\left(z\right)-2z^{3}\right)^{2}\right]}{\min_{z}\left[\psi_{f,\boldsymbol{e}}^{''}\left(z\right)\right]}
\end{equation}

From this, we can approximate the density under $f$ of the random
variable $\boldsymbol{e}$, which requires an intractable marginalization
of the variable $z$. Since we have an upper-bound of the KL divergence,
we can use the Evidence Lower-Bound (\citet{murphy2012machine} chapter
21) to approximate the marginalization and have an upper-bound on
the error:
\begin{equation}
\log\left[f\left(\boldsymbol{e}\right)\right]\approx C+E_{g}\left[\frac{z^{4}}{2}-\psi_{f,\boldsymbol{e}}\left(z\right)\right]
\end{equation}
We can finally turn this approximation of the density of $\boldsymbol{e}$
under model into an approximation of the KL divergence between $g$
and $f$ over this variable, as the variance of the ELBO as we sample
$\boldsymbol{e}$ from the density $g$:
\begin{equation}
KL\left(\boldsymbol{e}_{g},\boldsymbol{e}_{f}\right)\approx\text{var}\left[\boldsymbol{e}_{g}\rightarrow\left(E_{g}\left[\frac{z^{4}}{2}-\psi_{f,\boldsymbol{e}_{g}}\left(z\right)\right]\right)\right]
\end{equation}

The total KL divergence between $g$ and $f$ is then found as the
sum of the KL divergence caused by the $\boldsymbol{e}$ random-variable
and the mean KL divergence caused by the conditional $z|\boldsymbol{e}$:
\begin{equation}
KL\left(g,f\right)=KL\left(\boldsymbol{e}_{g},\boldsymbol{e}_{f}\right)+E\left[KL\left(z_{g}|\boldsymbol{e}_{g},z_{f}|\boldsymbol{e}_{g}\right)\right]
\end{equation}

\paragraph{Measuring the derivatives.}

In our theorem, two key quantities control the size of the KL divergence
between approximation and truth:
\begin{enumerate}
\item The KL divergence caused by the conditionals $KL\left(z_{g}|\boldsymbol{e}_{g},z_{f}|\boldsymbol{e}_{g}\right)$
is controlled by a LSI.
\item The KL divergence caused by $\boldsymbol{e}$ is controlled by approximating
the density $f\left(\boldsymbol{e}\right)$ which is done through
an Evidence Lower-Bound approximation.
\end{enumerate}
In both cases, we need to deal with expected values of differences
of the log-densities: $\phi_{f}-\phi_{g}$, where $\phi_{g}$ is a
Taylor expansion of $\phi_{f}$ to second-order. Thus, if we measure
the strength of the higher-derivatives of $\phi_{f}$, we can use
it to deduce the size of the difference: $\phi_{f}-\phi_{g}$, and
to control the KL divergence. Similarly, by controlling the higher-derivatives,
we are able to bound the minimum curvature of $z|\boldsymbol{e}$
under density $f$: $\min_{z}\left[\psi_{f,\boldsymbol{e}}^{''}\left(z\right)\right]$.

Furthermore, notice that the derivatives of $\phi_{f}$ only matter
once we fix a direction $\boldsymbol{e}$. Thus, we only need to measure
the size of the derivatives of $\phi_{f}\left(\boldsymbol{\theta}\right)$
along lines which go through $\boldsymbol{\theta}^{\star}$. In this
article, we will consider only derivatives up to fourth order, but
our result could be extended to involve higher-derivatives to any
arbitrary order. We define the following quantities:\begin{subequations}
\begin{align}
\Delta_{3}\left(\boldsymbol{e}\right) & =\frac{\partial^{3}}{\partial r^{3}}\left[\phi_{f}\left(\boldsymbol{\theta}^{\star}+r\left(H\phi_{f}\left(\boldsymbol{\theta}^{\star}\right)\right)^{-1/2}\boldsymbol{e}\right)\right]_{r=0}\label{eq: definition of delta3}\\
\Delta_{4}\left(\boldsymbol{e}\right) & =\max_{r\geq0}\left|\frac{\partial^{4}}{\partial r^{4}}\left[\phi_{f}\left(\boldsymbol{\theta}^{\star}+r\left(H\phi_{f}\left(\boldsymbol{\theta}^{\star}\right)\right)^{-1/2}\boldsymbol{e}\right)\right]\right|\label{eq: definition of delta4}
\end{align}
\end{subequations}

Second, notice that these derivatives impact the final result through
an average over the random variable $\boldsymbol{e}_{g}$, whether
as the size of the typical oscillation of the ELBO approximation of
$\log\left[f\left(\boldsymbol{e}\right)\right]$, or as the size of
the mean of $KL\left(z_{g}|\boldsymbol{e}_{g},z_{f}|\boldsymbol{e}_{g}\right)$.
Thus, instead of having to compute the maximum of $\Delta_{3}\left(\boldsymbol{e}\right)$
(which we could only do in non-polynomial time; \citet{hillar2013most}),
we can simply sample from $\boldsymbol{e}_{g}$ and compute an empirical
mean of $\Delta_{3}\left(\boldsymbol{e}\right)$ instead, which is
computationally much cheaper.

We are now finally ready to state our theorem.
\begin{thm}
A computable Bernstein-von Mises theorem.

If the higher log-derivatives of $f$ are small, as measured by eqs.
\ref{eq: definition of delta3},\ref{eq: definition of delta4}, $g$
is a good approximation of $f$:
\begin{itemize}
\item It gives a good approximation of the conditional density of $z|\boldsymbol{e}$:
\begin{equation}
KL\left(z_{g}|\boldsymbol{e},z_{f}|\boldsymbol{e}\right)\leq\frac{E_{g}\left[\left(\psi_{f,\boldsymbol{e}}^{'}\left(z\right)-2z^{3}\right)^{2}\right]}{\min_{z}\left[\psi_{f,\boldsymbol{e}}^{''}\left(z\right)\right]}\approx\frac{2\left[\Delta_{3}\left(\boldsymbol{e}\right)\right]^{2}}{\sqrt{3}\sqrt{2d-1}}\frac{\Gamma\left[\left(d+3\right)/2\right]}{\Gamma\left[d/2\right]}
\end{equation}
\item It gives a good approximation of the marginal density of $\boldsymbol{e}$:
\begin{align}
\log\left[f\left(\boldsymbol{e}\right)\right] & \approx C+E_{g}\left[\frac{z^{4}}{2}-\psi_{f,\boldsymbol{e}}\left(z\right)\right]\approx C+\frac{-\left(\sqrt{2}\right)\Delta_{3}\left(\boldsymbol{e}\right)}{3}\frac{\Gamma\left[\left(d+1\right)/2\right]}{\Gamma\left[d/2\right]}\\
KL\left(\boldsymbol{e}_{g},\boldsymbol{e}_{f}\right) & \approx\frac{1}{2}\text{var}\left[\boldsymbol{e}_{g}\rightarrow\log\left[f\left(\boldsymbol{e}_{g}\right)\right]\right]\\
 & \approx\frac{1}{9}\left[\frac{\Gamma\left[\left(d+1\right)/2\right]}{\Gamma\left[d/2\right]}\right]^{2}E\left[\Delta_{3}\left(\boldsymbol{e}_{g}\right)^{2}\right]
\end{align}
\item And thus gives a good global approximation:
\begin{equation}
KL\left(g,f\right)\lessapprox E\left[\Delta_{3}\left(\boldsymbol{e}_{g}\right)^{2}\right]\left(\frac{2}{\sqrt{3}\sqrt{2d-1}}\frac{\Gamma\left[\left(d+5\right)/2\right]}{\Gamma\left[d/2\right]}+\frac{1}{9}\left[\frac{\Gamma\left[\left(d+3\right)/2\right]}{\Gamma\left[d/2\right]}\right]^{2}\right)\label{eq: the approximate lower-bound}
\end{equation}
\end{itemize}
Please refer to our appendix for more detailed statements of this
theorem.

\textbf{Critically}, note that the approximate expressions correspond
to a rough first order approximation of the full bound which might
not always be appropriate.
\end{thm}

\paragraph{Computability.}

Critically, our upper-bound for the KL divergence is computable. In
its approximate version, the only term that depends on the target
density $\tilde{f}\left(\boldsymbol{\theta}\right)$ is the term $E\left[\Delta_{3}\left(\boldsymbol{e}_{g}\right)^{2}\right]$.
For each value of $\boldsymbol{e}$, $\Delta_{3}\left(\boldsymbol{e}\right)$
can be computed from the third-derivative tensor $\phi_{f}^{\left(3\right)}\left(\boldsymbol{\theta}^{\star}\right)$.
We are then left with approximating the expected value, which we can
do by sampling from $\boldsymbol{e}_{g}$. We could also find an explicit
formula for this expected value, though this approach wouldn't work
with the more exact forms of the bound (please refer to the appendix
for the detailed expression of the bound).

\paragraph{Example: logistic regression.}

In order to demonstrate the applicability of our bound, we now show
how it could be applied for linear logistic regression. If the data
$\mathcal{D}$ is composed of the $n$ pairs: $\left(y_{i},\boldsymbol{x}_{i}\right)\in\left\{ -1,1\right\} *\mathbb{R}^{d}$
and the prior is Gaussian (with mean 0 and covariance matrix $\sigma_{0}^{2}I_{d}$),
the posterior is:
\begin{equation}
p\left(\boldsymbol{\theta}|\mathcal{D}\right)=\frac{1}{Z}\exp\left(-\frac{1}{2}\sum_{j=1}^{d}\frac{\left(\theta_{i}\right)^{2}}{\sigma_{0}^{2}}\right)\prod_{i=1}^{n}\frac{1}{1+\exp\left[-y_{i}\left(\boldsymbol{\theta}.\boldsymbol{x}_{i}\right)\right]}
\end{equation}
which is log-concave, with very regular higher-derivatives. We compared
the approximate bound based on the third derivative (eq. \ref{eq: the approximate lower-bound})
and the true KL divergence (computed by sampling methods; see appendix).

In the following table, we report the value of the KL divergence,
of the approximate bound, and of their ratio in five typical examples.
As can be observed, the bound is fairly tight: the ratio between the
real KL divergence and the bound is consistently above $0.4$.
\begin{center}
\begin{tabular}{|c|c|c||c|c||c|}
\hline 
$d$ & $n$ & $\sigma_{0}$ & True $KL$ & Approximate bound & Bound efficiency\tabularnewline
\hline 
\hline 
5 & 20 & 10 & 0.31 & 0.82 & 0.38\tabularnewline
\hline 
5 & 100 & 10 & 0.057 & 0.11 & 0.5\tabularnewline
\hline 
5 & 1000 & 10 & 0.0065 & 0.011 & 0.55\tabularnewline
\hline 
50 & 100 & 10 & 528 & 1288 & 0.41\tabularnewline
\hline 
50 & 1000 & 10 & 0.38 & 0.46 & 0.83\tabularnewline
\hline 
\end{tabular}
\par\end{center}

\section{Conclusion}

We have presented an upper-bound for the KL divergence between the
Laplace approximation and the true posterior distribution. We have
applied successfully this bound to the classical logistic regression
model and we have observed that it is tight even if we only use the
terms which depend on the third derivative.

The theorem we have presented here represents only a small fraction
of a longer article which we will submit to the Annals of Statistics.
In this new article, we will present a more general version of the
bound which can be applied to any Gaussian approximation. We will
discuss why and when this bound is tight. We will also prove, through
a frequentist analysis of the posterior, why and when the approximation
with only the third derivative is valid. Finally, we will discuss
how to extend this line of work to non log-concave approximations
through other shape assumptions.

This work represents a major step forward for Gaussian approximation
methods in Bayesian inference as it will enable the statistician to
characterize, on a given problem, how good his approximation is through
one cheap extra calculation. Should the approximation prove to be
poor, he will then choose a more precise but more expensive approximation
method.\newpage{}

\bibliographystyle{plainnat}
\bibliography{ref}

\newpage{}

\appendix

\part*{Appendix: proofs and detailed statement of the theorem}

In this appendix, we give detailed versions of the theorem that we
have stated in our main article, as well as a detailed proof of this
theorem and of various important intermediate lemmas (section \ref{sec:Upper-bounding-the-KL}).
We also describe precisely how our example with logistic regression
was handled (section \ref{sec:Details-of-the logistic regression}).

A detailed statement of our theorem can be found as Theorem \vref{thm: Detailed theorem}.

\section{Upper-bounding the KL divergence\label{sec:Upper-bounding-the-KL}}

This section focuses on proving our theorem that upper bounds the
KL divergence between the Laplace approximation $g\left(\boldsymbol{\theta}\right)$
and the target density $f\left(\boldsymbol{\theta}\right)$.

Throughout this section, we will denote with $\boldsymbol{\theta}_{g}$
a Gaussian random variable with density $g\left(\boldsymbol{\theta}\right)$
and with $\boldsymbol{\theta}_{f}$ a random variable with density
$f\left(\boldsymbol{\theta}\right)$. When we perform our change of
variable, we will use the subscript to indicate which density is used.

We first detail our change of variable: $\boldsymbol{\theta}\rightarrow\left(r,\boldsymbol{e}\right)\rightarrow\left(z,\boldsymbol{e}\right)$,
and next how it ensures that $\tilde{f}\left(z|\boldsymbol{e}\right)$
is strongly log-concave. We then show how the LSI can be used to upper-bound
the KL divergence between $z_{g}$ and $z_{f}|\boldsymbol{e}$ for
any value of $\boldsymbol{e}$. We next bound the divergence between
$\boldsymbol{e}_{g}$ and $\boldsymbol{e}_{f}$. Finally, we state
our theorem that upper-bounds the KL divergence between $g$ and $f$. 

\subsection{Two changes of variable}

Our objective is to prove that the two random variables $\boldsymbol{\theta}_{g}$
and $\boldsymbol{\theta}_{f}$ have almost the same distribution,
in that the KL divergence $KL\left(g,f\right)$ is close to 0.

In order to prove that, we will use a change of variable that considerably
simplifies our work. Let $\boldsymbol{\theta}^{\star}$ be the minimum
of $\phi_{f}\left(\boldsymbol{\theta}\right)=-\log\left(\tilde{f}\left(\boldsymbol{\theta}\right)\right)$,
and $\Sigma=\left[H\phi_{f}\left(\boldsymbol{\theta}^{\star}\right)\right]^{-1}$.
Since $g\left(\boldsymbol{\theta}\right)$ is the Laplace approximation
of $f\left(\boldsymbol{\theta}\right)$, $\boldsymbol{\theta}^{\star}$
is the mean of $\boldsymbol{\theta}_{g}$ and $\Sigma$ its covariance.
We can thus perform a fist change of variable to ``standardize''
$\boldsymbol{\theta}_{g}$, i.e: re-express it as a translated and
scaled version of the standard Gaussian variable $\boldsymbol{\eta}$,
with mean 0 and covariance the identity matrix:
\begin{align*}
\boldsymbol{\theta}_{g} & =\boldsymbol{\theta}^{\star}+\left[H\phi_{f}\left(\boldsymbol{\theta}^{\star}\right)\right]^{-1/2}\boldsymbol{\eta}\\
 & =\boldsymbol{\theta}^{\star}+\Sigma^{1/2}\boldsymbol{\eta}
\end{align*}
We then perform two further changes of variable on the standard Gaussian
$\boldsymbol{\eta}$. We first re-express it as a product of a radius
$r_{g}$ and a direction $\boldsymbol{e}_{g}$. $r_{g}$ takes values
inside $\mathbb{R}_{+}$ and $\boldsymbol{e}$ takes values inside
the $d$-dimensional unit sphere:
\begin{align*}
\boldsymbol{\eta} & =r_{g}\boldsymbol{e}_{g}\\
r_{g} & =\left\Vert \boldsymbol{\eta}\right\Vert \\
\boldsymbol{e}_{g} & =\frac{\boldsymbol{\eta}}{\left\Vert \boldsymbol{\eta}\right\Vert }
\end{align*}
Since $\boldsymbol{\eta}$ is a standard Gaussian distribution, it
is symmetric. Thus the random variable $\boldsymbol{e}_{g}$ is uniformly
distributed over the $d$-dimensional sphere: $S^{d-1}$ (with density
the inverse of the area of $S^{d-1}$). $r_{g}$ also has a straightforward
distribution: its distribution is $\chi_{d}$: a \emph{chi-distribution}
(pronounced ``ki'') with $d$-degrees of freedom. Furthermore, random
variables $\boldsymbol{e}_{g}$ and $r_{g}$ are independent (see
Lemma \ref{lem:Density-of-. z_g e_g}).

Our final change of variable consists in remapping the random variable
$r_{g}$ using a simple bijection.
\[
r_{g}=z_{g}^{2}
\]

The bijection makes it so that large values of $r_{g}$ are compressed
which ensures, as we prove in the next section, that we end up with
a strongly log-concave density $\tilde{f}\left(z\right)$.

We perform the exact same change of variable on $\boldsymbol{\theta}_{f}$.
We thus transform the comparison of $\boldsymbol{\theta}_{g}$ and
$\boldsymbol{\theta}_{f}$ into that of the pairs $\left(z_{g},\boldsymbol{e}_{g}\right)$
and $\left(z_{f},\boldsymbol{e}_{f}\right)$:\begin{subequations}
\begin{align}
\boldsymbol{\theta}_{g} & =\boldsymbol{\theta}^{\star}+z_{g}^{2}\left(\Sigma^{1/2}\boldsymbol{e}_{g}\right)\\
\boldsymbol{\theta}_{f} & =\boldsymbol{\theta}^{\star}+z_{f}^{2}\left(\Sigma^{1/2}\boldsymbol{e}_{f}\right)
\end{align}
\end{subequations}

We will denote the log-density of $z_{g}$ and $z_{f}|\boldsymbol{e}$
using respectively $\psi_{g}$ and $\psi_{f,\boldsymbol{e}}$, i.e:
\begin{align*}
\psi_{g}\left(z_{g}\right) & =-\log\left[g\left(z_{g}\right)\right]\\
\psi_{f,\boldsymbol{e}}\left(z_{f}\right) & =-\log\left[f\left(z_{f}|\boldsymbol{e}\right)\right]
\end{align*}
Note that we compare the conditional random variable $z_{f}|\boldsymbol{e}$
to the marginal $z_{g}$ because $z_{g}$ and $\boldsymbol{e}_{g}$
are independent. The conditional random variable $z_{g}|\boldsymbol{e}$
and the marginal random variable $z_{g}$ are thus the same.

\subsection{Density of $\left(z,\boldsymbol{e}\right)$ under the two models}

Let us now study the density of the pair of random variables $\left(z,\boldsymbol{e}\right)$
under our two models.

\subsubsection{Density of $\left(z_{g},\boldsymbol{e}_{g}\right)$}

First, under the Gaussian density $g\left(\boldsymbol{\theta}\right)$,
these two variables are independent, and have simple densities, as
summed-up by the following lemma.
\begin{lem}
Density of $\left(z_{g},\boldsymbol{e}_{g}\right)$.\label{lem:Density-of-. z_g e_g}

The random variables $\left(z_{g},\boldsymbol{e}_{g}\right)$ are
independent.

The random variable $z_{g}$ follows a $\chi_{d}^{1/2}$ distribution
with density:
\begin{equation}
g\left(z_{g}\right)=\frac{1}{2^{d/2-2}\Gamma\left(d/2\right)}z^{2d-1}\exp\left(-\frac{z^{4}}{2}\right)
\end{equation}
The random variable $\boldsymbol{e}_{g}$ is uniformly distributed
over its support: the $d$-dimensional sphere $S^{d-1}$.
\end{lem}
Furthermore, the density $g\left(z_{g}\right)$ is strongly log-concave.
\begin{lem}
Strong log-concavity of $g\left(z_{g}\right)$.\label{lem:Strong-log-concavity-of z_g e_g}

The log-density $\psi_{g}\left(z_{g}\right)$ is strongly concave:
\[
\min_{z}\left[\psi_{g}^{''}\left(z\right)\right]=2\sqrt{6}\sqrt{2d-1}
\]
\end{lem}
\begin{proof}
Let us now prove both of these lemma.

First, we compute the density of $r,\boldsymbol{e}$ from the density
of the standard Gaussian $\boldsymbol{\eta}$. This just follows from
a straightforward change of variable formula:
\begin{align*}
g\left(r,\boldsymbol{e}\right) & \propto r^{d-1}\exp\left(-\frac{1}{2}\left(r\boldsymbol{e}\right)^{T}I_{d}\left(r\boldsymbol{e}\right)\right)1\left(r\in\mathbb{R}_{+}\right)1\left(\boldsymbol{e}\in S^{d-1}\right)\\
 & \propto r^{d-1}\exp\left(-\frac{r^{2}}{2}\right)1\left(r\in\mathbb{R}_{+}\right)1\left(\boldsymbol{e}\in S^{d-1}\right)
\end{align*}
We observe that this decomposes into a product of the two marginal
densities. We further observe that $r_{g}$ follows a $\chi_{d}$
distribution, for which we know the normalization constant to be $\frac{1}{2^{d/2-1}\Gamma\left(d/2\right)}$.
A further change of variable yields the density of $z=\sqrt{r}$:
\begin{align*}
g\left(z\right) & =\frac{1}{2^{d/2-1}\Gamma\left(d/2\right)}\left(z^{2}\right)^{d-1}\exp\left(-\frac{z^{4}}{2}\right)2z\\
 & =\frac{1}{2^{d/2-2}\Gamma\left(d/2\right)}z^{2d-1}\exp\left(-\frac{z^{4}}{2}\right)
\end{align*}

Second, let us study the log-concavity of this density. The negative
log-density (which we need to show is strongly convex) is:
\[
\psi_{g}\left(z\right)=-\left(2d-1\right)\log z+\frac{z^{4}}{2}+\log\left[2^{d/2-2}\Gamma\left(d/2\right)\right]
\]

The derivatives of $\psi_{g}\left(z\right)$ are straightforward to
compute:
\begin{align*}
\psi_{g}^{'}\left(z\right) & =-\frac{2d-1}{z}+2z^{3}\\
\psi_{g}^{''}\left(z\right) & =\frac{2d-1}{z^{2}}+6z^{2}\\
\psi_{g}^{\left(3\right)}\left(z\right) & =-2\frac{2d-1}{z^{3}}+12z
\end{align*}
Observe that $\psi_{g}^{''}\left(z\right)>0$ so that $\psi_{g}$
is at least strictly convex. Furthermore, $\psi_{g}^{\left(4\right)}\left(z\right)>0$
so that $\psi_{g}^{''}\left(z\right)$ is also strictly convex. It
thus reaches its unique minimum at the point for which $\psi_{g}^{\left(3\right)}\left(z\right)=0$.
This point is such that:
\begin{align*}
-2\frac{2d-1}{z^{3}}+12z & =0\\
z^{4} & =\frac{2d-1}{6}\\
z & =\left(\frac{2d-1}{6}\right)^{1/4}
\end{align*}
Thus, the minimum curvature of $\psi_{g}\left(z\right)$ is:
\begin{align}
\min_{z}\left[\psi_{g}^{''}\left(z\right)\right] & =\frac{2d-1}{\left(\frac{2d-1}{6}\right)^{1/2}}+6\left(\frac{2d-1}{6}\right)^{1/2}\nonumber \\
 & =6^{1/2}\left(2d-1\right)^{1/2}+6^{1/2}\left(2d-1\right)^{1/2}\nonumber \\
 & =2\sqrt{6}\sqrt{2d-1}
\end{align}

which concludes our proof.
\end{proof}

\subsubsection{Density of $\left(z_{f},\boldsymbol{e}_{f}\right)$}

We now investigate the density of the pair $\left(z,\boldsymbol{e}\right)$
under the target density $f$. The best description for this pair
of variable is a hierarchical description in which the direction $\boldsymbol{e}_{f}$
is picked first according to its marginal density $f\left(\boldsymbol{e}_{f}\right)$.
The ``square-root-radius'' $z_{f}$ (or equivalently, the radius
$r_{f}$) is then picked according to its conditional distribution
$z_{f}|\boldsymbol{e}_{f}$.

Our first lemma of this section describes the density marginal density
of $\boldsymbol{e}_{f}$ and the conditional density $z_{f}|\boldsymbol{e}_{f}$.
\begin{lem}
Density of $\left(z_{f},\boldsymbol{e}_{f}\right)$. \label{lem:Density-of-. z_g e_f}

The conditional density of $z_{f}|\boldsymbol{e}_{f}$ is:
\[
f\left(z_{f}|\boldsymbol{e}_{f}\right)\propto\left(z_{f}\right)^{2d-1}\exp\left(-\phi_{f}\left(\boldsymbol{\theta}^{\star}+z^{2}\Sigma^{1/2}\boldsymbol{e}_{f}\right)\right)
\]

The marginal distribution of $\boldsymbol{e}_{f}$ is found by integrating
out the conditional distribution of either $z_{f}|\boldsymbol{e}$
or $r|\boldsymbol{e}$:
\begin{align*}
f\left(\boldsymbol{e}_{f}\right) & \propto\int_{z\geq0}\left(z_{f}\right)^{2d-1}\exp\left(-\phi_{f}\left(\boldsymbol{\theta}^{\star}+z^{2}\Sigma^{1/2}\boldsymbol{e}_{f}\right)\right)dz_{f}\\
 & \propto\int_{r\geq0}\left(r_{f}\right)^{d-1}\exp\left(-\phi_{f}\left(\boldsymbol{\theta}^{\star}+r\Sigma^{1/2}\boldsymbol{e}_{f}\right)\right)dr
\end{align*}
\end{lem}
Furthermore, we can now finally highlight why the change of variable
$r=z^{2}$ is important: as the next lemma asserts, this change of
variable ensures that $f\left(z|\boldsymbol{e}_{f}\right)$ is always
strongly log-concave. In the limit where the higher derivatives of
$\phi_{f}\left(\boldsymbol{\theta}\right)$ along direction $\boldsymbol{e}_{f}$
become negligible, the minimum log-curvature of $f\left(z|\boldsymbol{e}_{f}\right)$
even asymptotes to $\min_{z}\psi^{''}\left(z\right)$.

In order to state this lemma, we will need additional notation. First,
we need a shorter notation for the function $r\rightarrow\phi_{f}\left(\boldsymbol{\theta}^{\star}+r\Sigma^{1/2}\boldsymbol{e}\right)$
which we will denote with $\varphi_{\boldsymbol{e}}\left(r\right)$.
Second, we need to measure the derivatives of this function. Notice
that we already know the first two derivatives of $\varphi_{\boldsymbol{e}}\left(r\right)$:
\begin{align*}
\varphi_{\boldsymbol{e}}^{'}\left(r\right) & =\nabla\phi_{f}\left(\boldsymbol{\theta}^{\star}\right)\Sigma^{1/2}\boldsymbol{e}=0\\
\varphi_{\boldsymbol{e}}^{''}\left(r\right) & =\boldsymbol{e}^{T}\Sigma^{1/2}H\phi_{f}\left(\boldsymbol{\theta}^{\star}\right)\Sigma^{1/2}\boldsymbol{e}\\
 & =\boldsymbol{e}^{T}\Sigma^{1/2}\Sigma^{-1}\Sigma^{1/2}\boldsymbol{e}\\
 & =\boldsymbol{e}^{T}I_{d}\boldsymbol{e}\\
 & =1
\end{align*}
We will control the higher-derivatives using the following two quantities:\begin{subequations}
\begin{align}
\Delta_{3}\left(\boldsymbol{e}\right) & =\varphi_{\boldsymbol{e}}^{\left(3\right)}\left(0\right)\\
\Delta_{4}\left(\boldsymbol{e}\right) & =\max_{r\geq0}\left[\varphi_{\boldsymbol{e}}^{\left(4\right)}\left(r\right)\right]
\end{align}
\end{subequations}Note that $\Delta_{3}\left(\boldsymbol{e}\right)$
can be deduced from the third-derivative tensor: $\phi_{f}^{\left(3\right)}\left(\boldsymbol{\theta}^{\star}\right)$
through:
\begin{equation}
\Delta_{3}\left(\boldsymbol{e}\right)=\phi_{f}^{\left(3\right)}\left(\boldsymbol{\theta}^{\star}\right)\left[\Sigma^{1/2}\boldsymbol{e},\Sigma^{1/2}\boldsymbol{e},\Sigma^{1/2}\boldsymbol{e}\right]
\end{equation}

\begin{lem}
Strong log-concavity of $f\left(z_{f}|\boldsymbol{e}\right)$. \label{lem:Strong-log-concavity-of z_g | e_f}

For any $\boldsymbol{e}$, the conditional density $f\left(z_{f}|\boldsymbol{e}\right)$
is strongly log-concave. The minimum curvature can be found numerically
from the following formula:
\begin{align*}
r_{0} & =\frac{\Delta_{3}\left(\boldsymbol{e}\right)+\sqrt{\left[\Delta_{3}\left(\boldsymbol{e}\right)\right]^{2}+2\Delta_{4}\left(\boldsymbol{e}\right)}}{\Delta_{4}\left(\boldsymbol{e}\right)}\\
\text{min curvature} & =\min\left\{ r_{0}+\Delta_{3}\left(\boldsymbol{e}\right)\left(r_{0}\right)^{2}-\Delta_{4}\left(\boldsymbol{e}\right)\frac{\left(r_{0}\right)^{3}}{3}\ ;\ \min_{0\leq r\leq r_{0}}\frac{2d-1}{r}+6r+5\Delta_{3}\left(\boldsymbol{e}\right)r^{2}-\frac{7}{3}\Delta_{4}\left(\boldsymbol{e}\right)r^{4}\right\} 
\end{align*}
\end{lem}
\begin{proof}
Now let us prove these two lemmas.

The first lemma is absolutely straightforward: it results from a straightforward
change of variable formula.

The second lemma is more difficult. First, let us start by computing
the derivatives of $\psi_{f,\boldsymbol{e}}\left(z\right)$:
\begin{align}
\psi_{f,\boldsymbol{e}}\left(z\right) & =-\left(2d-1\right)\log\left(z\right)+\varphi_{\boldsymbol{e}}\left(z^{2}\right)\nonumber \\
\psi_{f,\boldsymbol{e}}^{'}\left(z\right) & =-\frac{2d-1}{z}+2z\varphi_{\boldsymbol{e}}^{'}\left(z^{2}\right)\nonumber \\
\psi_{f,\boldsymbol{e}}^{''}\left(z\right) & =\frac{2d-1}{z^{2}}+2\varphi_{\boldsymbol{e}}^{'}\left(z^{2}\right)+4z^{2}\varphi_{\boldsymbol{e}}^{''}\left(z^{2}\right)
\end{align}
Let us consider the equation for $\psi_{f,\boldsymbol{e}}^{''}\left(z\right)$.
All of the terms are positive:
\begin{itemize}
\item The first term $\left(2d-1\right)/z^{2}>0$ is obvious
\item The second term is also positive: since $\varphi_{\boldsymbol{e}}^{''}\left(r\right)>0$,
we have $\varphi_{\boldsymbol{e}}^{'}\left(r\right)>\varphi_{\boldsymbol{e}}^{'}\left(0\right)=0$
\item The final term is also positive: $\varphi_{\boldsymbol{e}}^{''}\left(r\right)>0$
so that $4z^{2}\varphi_{\boldsymbol{e}}^{''}\left(z^{2}\right)>0$
\end{itemize}
Furthermore, it is even true that $\psi_{f,\boldsymbol{e}}^{''}\left(z\right)$
is lower-bounded, so that $\psi_{f,\boldsymbol{e}}^{''}\left(z\right)$
is strongly log-concave. Indeed, we have:
\begin{equation}
\psi_{f,\boldsymbol{e}}^{''}\left(z\right)>\frac{2d-1}{z^{2}}+2\varphi_{\boldsymbol{e}}^{'}\left(z^{2}\right)
\end{equation}
Critically, the first term is decreasing, and the second one increasing.
For any $z_{0}$, we have that, for all $z\geq z_{0}$:
\begin{align*}
\frac{2d-1}{z^{2}}+2\varphi_{\boldsymbol{e}}^{'}\left(z^{2}\right) & \geq\frac{2d-1}{z^{2}}+2\varphi_{\boldsymbol{e}}^{'}\left(z_{0}^{2}\right)\\
\psi_{f,\boldsymbol{e}}^{''}\left(z\right) & >2\varphi_{\boldsymbol{e}}^{'}\left(z_{0}^{2}\right)
\end{align*}
Thus, for any $z_{0}$, we have the following strictly positive lower
bound for $\psi_{f,\boldsymbol{e}}^{''}\left(z\right)$:
\begin{equation}
\psi_{f,\boldsymbol{e}}^{''}\left(z\right)\geq\min\left\{ \min_{0\leq z\leq z_{0}}\left[\frac{2d-1}{z^{2}}+2\varphi_{\boldsymbol{e}}^{'}\left(z^{2}\right)+4z^{2}\varphi_{\boldsymbol{e}}^{''}\left(z^{2}\right)\right];2\varphi_{\boldsymbol{e}}^{'}\left(z_{0}^{2}\right)\right\} >0
\end{equation}
We have just proved that $f\left(z_{f}|\boldsymbol{e}\right)$ is
strongly log-concave. Let us now see how we should choose the value
for $z_{0}$.

We do so by first computing a Taylor expansion of $\varphi_{\boldsymbol{e}}^{''}\left(r\right)$:
\begin{equation}
\varphi_{\boldsymbol{e}}^{''}\left(r\right)\geq1+\Delta_{3}\left(\boldsymbol{e}\right)r-\Delta_{4}\left(\boldsymbol{e}\right)\frac{r^{2}}{2}
\end{equation}
However, we have access to further information: we know that $\varphi_{\boldsymbol{e}}$
is strictly concave so that $\varphi_{\boldsymbol{e}}^{''}\left(r\right)>0$.

Thus, there will be a critical value $r_{0}$ such that the lower-bound
computed from the Taylor expansion is equal to 0. For $r\ge r_{0}$,
the Taylor expansion bound gives no further information compared to
simply knowing that $\varphi_{\boldsymbol{e}}^{''}\left(r\right)>0$.
This value $r_{0}$ is found by solving a second degree polynomial,
yielding:
\begin{equation}
r_{0}=\frac{\Delta_{3}\left(\boldsymbol{e}\right)+\sqrt{\left[\Delta_{3}\left(\boldsymbol{e}\right)\right]^{2}+2\Delta_{4}\left(\boldsymbol{e}\right)}}{\Delta_{4}\left(\boldsymbol{e}\right)}
\end{equation}
This gives the limit of the zone for which our Taylor expansion is
useful.

Now let us compute a Taylor expansion of $\varphi_{\boldsymbol{e}}^{'}\left(r\right)$.
In the useful region $r\leq r_{0}$, we have:
\begin{equation}
\varphi_{\boldsymbol{e}}^{'}\left(r\right)\geq0+r+\Delta_{3}\left(\boldsymbol{e}\right)\frac{r^{2}}{2}-\Delta_{4}\left(\boldsymbol{e}\right)\frac{r^{3}}{3!}
\end{equation}
For $r\geq r_{0}$, the only guarantee we have is that $\varphi_{\boldsymbol{e}}^{'}\left(r\right)$
is increasing so that:
\begin{equation}
\varphi_{\boldsymbol{e}}^{'}\left(r\right)\geq\varphi_{\boldsymbol{e}}^{'}\left(r_{0}\right)\geq r_{0}+\Delta_{3}\left(\boldsymbol{e}\right)\frac{\left(r_{0}\right)^{2}}{2}-\Delta_{4}\left(\boldsymbol{e}\right)\frac{\left(r_{0}\right)^{3}}{3!}
\end{equation}

We now combine these lower-bounds on $\varphi_{\boldsymbol{e}}^{'}\left(r\right)$
and $\varphi_{\boldsymbol{e}}^{''}\left(r\right)$ with the expression
for $\psi_{f,\boldsymbol{e}}^{''}\left(z\right)$. For $z\leq z_{0}=\sqrt{r_{0}}$,
we have:
\begin{align}
\psi_{f,\boldsymbol{e}}^{''}\left(z\right) & \geq\frac{2d-1}{z^{2}}+2z^{2}+\Delta_{3}\left(\boldsymbol{e}\right)z^{4}-\Delta_{4}\left(\boldsymbol{e}\right)\frac{z^{6}}{3}+4z^{2}+\Delta_{3}\left(\boldsymbol{e}\right)4z^{4}-\Delta_{4}\left(\boldsymbol{e}\right)2z^{6}\nonumber \\
 & \geq\frac{2d-1}{z^{2}}+6z^{2}+5\Delta_{3}\left(\boldsymbol{e}\right)z^{4}-\frac{7}{3}\Delta_{4}\left(\boldsymbol{e}\right)z^{6}
\end{align}
and for $z\geq z_{0}$, we only have the somewhat trivial bound:
\begin{align}
\psi_{f,\boldsymbol{e}}^{''}\left(z\right) & >2\varphi_{\boldsymbol{e}}^{'}\left(z_{0}^{2}\right)\nonumber \\
\psi_{f,\boldsymbol{e}}^{''}\left(z\right) & \geq r_{0}+\Delta_{3}\left(\boldsymbol{e}\right)\left(r_{0}\right)^{2}-\Delta_{4}\left(\boldsymbol{e}\right)\frac{\left(r_{0}\right)^{3}}{3}
\end{align}

At this point, we have an expression that is perfectly suitable for
numerical optimization: we simply need to compute the extrema of a
polynomial function over a finite range. The expression in the theorem
is reached through the change of variable $r=z^{2}$.
\end{proof}

\subsection{Approximating $z_{f}$}

We now turn to the task of computing whether the random variable $z_{g}$
is a good approximation of $z_{f}|\boldsymbol{e}$. More precisely,
since we have proved that $f\left(z|\boldsymbol{e}\right)$ is strongly
log-concave, we can apply the Log-Sobolev Inequality (LSI; \citet{otto2000generalization})
to upper-bound the KL divergence while avoiding the complicated task
of upper-bounding the normalizing constant of $\tilde{f}\left(z|\boldsymbol{e}\right)$.

The following lemma gives the result of applying the LSI. We express
the results using properties of the distribution of the random variable
$r_{g}$, which follows the more common $\chi_{d}$ distribution,
instead of $z_{g}$. Moments of a $\chi_{d}$ random variable can
be found in any thorough reference textbook on probability theory.
\begin{lem}
$z_{g}\approx z_{f}|\boldsymbol{e}$ \label{lem: z_g approx z_f | e}

The KL divergence between $g\left(z\right)$ and $f\left(z|\boldsymbol{e}\right)$
is upper-bounded:
\begin{align*}
KL\left(z_{g},z_{f}|\boldsymbol{e}\right) & \leq\frac{E\left[4r_{g}\left(\varphi_{\boldsymbol{e}}^{'}\left(r_{g}\right)-r_{g}\right)^{2}\right]}{{\displaystyle \min_{z\geq0}\left[\psi_{f,\boldsymbol{e}}^{''}\left(z\right)\right]}}\\
 & \leq\frac{\left[\Delta_{3}\left(\boldsymbol{e}\right)\right]^{2}E\left(r_{g}^{5}\right)+\frac{2}{3}\left|\Delta_{3}\left(\boldsymbol{e}\right)\right|\Delta_{4}\left(\boldsymbol{e}\right)E\left(r_{g}^{6}\right)+\frac{1}{9}\left[\Delta_{4}\left(\boldsymbol{e}\right)\right]^{2}E\left(r_{g}^{7}\right)}{{\displaystyle \min_{z\geq0}\left[\psi_{f,\boldsymbol{e}}^{''}\left(z\right)\right]}}\\
 & \lessapprox\frac{\left[\Delta_{3}\left(\boldsymbol{e}\right)\right]^{2}E\left(r_{g}^{5}\right)}{2\sqrt{6}\sqrt{2d-1}}
\end{align*}
\end{lem}
\begin{proof}
This lemma is proved by a simple combination of the LSI with a Taylor
expansion of $\varphi_{\boldsymbol{e}}^{'}\left(r\right)$ around
$0$.

First, we observe that $f\left(z|\boldsymbol{e}\right)$ is a strongly
log-concave density with minimal curvature $\min\left[\psi_{f,\boldsymbol{e}}^{''}\left(z\right)\right]$
(from lemma \ref{lem:Strong-log-concavity-of z_g | e_f}). Thus, we
can apply the LSI:
\begin{align*}
KL\left(z_{g},z_{f}|\boldsymbol{e}\right) & \leq\frac{E\left[\left(\psi_{f,\boldsymbol{e}}^{'}\left(z_{g}\right)-\psi_{g}^{'}\left(z_{g}\right)\right)^{2}\right]}{{\displaystyle \min_{z\geq0}\left[\psi_{f,\boldsymbol{e}}^{''}\left(z\right)\right]}}\\
 & \leq\frac{E\left[4z_{g}^{2}\left(\varphi_{\boldsymbol{e}}^{'}\left(z_{g}^{2}\right)-z_{g}^{2}\right)^{2}\right]}{{\displaystyle \min_{z\geq0}\left[\psi_{f,\boldsymbol{e}}^{''}\left(z\right)\right]}}\\
 & \leq\frac{E\left[4r_{g}\left(\varphi_{\boldsymbol{e}}^{'}\left(r_{g}\right)-r_{g}\right)^{2}\right]}{{\displaystyle \min_{z\geq0}\left[\psi_{f,\boldsymbol{e}}^{''}\left(z\right)\right]}}
\end{align*}
where the following lines correspond to simple substitutions: $\psi_{f,\boldsymbol{e}}^{'}\left(z\right)=2z\varphi_{\boldsymbol{e}}\left(z^{2}\right)$
and $r_{g}=z_{g}^{2}$.

We then turn to a Taylor expansion of $\varphi_{\boldsymbol{e}}^{'}\left(r_{g}\right)-r_{g}$
around 0. Critically, the first two terms are 0 because our Gaussian
approximation is the Laplace approximation:
\begin{align*}
\left|\varphi_{\boldsymbol{e}}^{'}\left(r_{g}\right)-r_{g}-\Delta_{3}\left(\boldsymbol{e}\right)\frac{r_{g}^{2}}{2}\right| & \leq\Delta_{4}\left(\boldsymbol{e}\right)\frac{r_{g}^{3}}{3!}\\
\left|\varphi_{\boldsymbol{e}}^{'}\left(r_{g}\right)-r_{g}\right| & \leq\left|\Delta_{3}\left(\boldsymbol{e}\right)\right|\frac{r_{g}^{2}}{2}+\Delta_{4}\left(\boldsymbol{e}\right)\frac{r_{g}^{3}}{3!}\\
\left(\varphi_{\boldsymbol{e}}^{'}\left(r_{g}\right)-r_{g}\right)^{2} & \leq\left(\left|\Delta_{3}\left(\boldsymbol{e}\right)\right|\frac{r_{g}^{2}}{2}+\Delta_{4}\left(\boldsymbol{e}\right)\frac{r_{g}^{3}}{3!}\right)^{2}\\
4r_{g}\left(\varphi_{\boldsymbol{e}}^{'}\left(r_{g}\right)-r_{g}\right)^{2} & \leq r_{g}\left(\left|\Delta_{3}\left(\boldsymbol{e}\right)\right|r_{g}^{2}+\Delta_{4}\left(\boldsymbol{e}\right)\frac{r_{g}^{3}}{3}\right)^{2}
\end{align*}
We can then easily compute the expected value of the last bound, yielding:
\[
E\left[4r_{g}\left(\varphi_{\boldsymbol{e}}^{'}\left(r_{g}\right)-r_{g}\right)^{2}\right]\leq\left[\Delta_{3}\left(\boldsymbol{e}\right)\right]^{2}E\left(r_{g}^{5}\right)+\frac{2}{3}\left|\Delta_{3}\left(\boldsymbol{e}\right)\right|\Delta_{4}\left(\boldsymbol{e}\right)E\left(r_{g}^{6}\right)+\frac{1}{9}\left[\Delta_{4}\left(\boldsymbol{e}\right)\right]^{2}E\left(r_{g}^{7}\right)
\]
which yields the claimed result.
\end{proof}

\subsection{Approximating $\boldsymbol{e}_{f}$}

We can now turn to the task of computing whether $\boldsymbol{e}_{g}$
gives a good approximation of $\boldsymbol{e}_{f}$. This corresponds
to checking whether $\boldsymbol{e}_{f}$ has an almost uniform distribution
over the unit sphere $S^{d-1}$. Equivalently, we will check whether
the log-density $\log\left[\tilde{f}\left(\boldsymbol{e}\right)\right]$
has small oscillations.

The value of $\tilde{f}\left(\boldsymbol{e}\right)$ is found by integrating
out the unnormalized density $\tilde{f}\left(r|\boldsymbol{e}\right)$
(or equivalently $\tilde{f}\left(z|\boldsymbol{e}\right)$). The oscillations
in $\xi\left(\boldsymbol{e}\right)=\log\left[\tilde{f}\left(\boldsymbol{e}\right)\right]$
are caused by the fact that the higher-derivatives of $\varphi_{\boldsymbol{e}}\left(r\right)$
differ depending on the direction $\boldsymbol{e}$: as the following
lemma shows, $\Delta_{3}\left(\boldsymbol{e}\right)$ is the main
influence on $\xi\left(\boldsymbol{e}\right)=\log\left[\tilde{f}\left(\boldsymbol{e}\right)\right]$.
\begin{lem}
Oscillations of $\log\left[\tilde{f}\boldsymbol{e}\right]$.\label{lem:Oscillations-of-xi(e)}

$\log\left[\tilde{f}\left(\boldsymbol{e}\right)\right]$ can be approximated
using the ELBO:
\[
\xi\left(\boldsymbol{e}\right)=\log\left[\tilde{f}\left(\boldsymbol{e}\right)\right]=C+E_{g}\left(\frac{r_{g}^{2}}{2}-\varphi_{\boldsymbol{e}}\left(r_{g}\right)\right)+\epsilon_{1}\left(\boldsymbol{e}\right)
\]
where $\epsilon_{1}\left(\boldsymbol{e}\right)$ is a positive error
equal precisely to $KL\left(z_{g},z_{f}|\boldsymbol{e}\right)$ and
thus upper-bounded by lemma \ref{lem: z_g approx z_f | e}:
\[
\epsilon_{1}\left(\boldsymbol{e}\right)=KL\left(z_{g},z_{f}|\boldsymbol{e}\right)\leq\frac{E\left[4r_{g}\left(\varphi_{\boldsymbol{e}}^{'}\left(r_{g}\right)-r_{g}\right)^{2}\right]}{{\displaystyle \min_{z\geq0}\left[\psi_{f,\boldsymbol{e}}^{''}\left(z\right)\right]}}
\]

We can further perform a Taylor expansion of $\varphi_{\boldsymbol{e}}\left(r_{g}\right)$
to get:
\begin{align*}
\xi\left(\boldsymbol{e}\right) & =C-\frac{\Delta_{3}\left(\boldsymbol{e}\right)E\left(r_{g}^{3}\right)}{6}+\epsilon_{1}\left(\boldsymbol{e}\right)+\epsilon_{2}\left(\boldsymbol{e}\right)\\
\left|\epsilon_{2}\left(\boldsymbol{e}\right)\right| & \leq\frac{\Delta_{4}\left(\boldsymbol{e}\right)E\left(r_{g}^{4}\right)}{4!}
\end{align*}
\end{lem}
If these derivatives are small on average, then the KL divergence
$KL\left(\boldsymbol{e}_{g},\boldsymbol{e}_{f}\right)$ is small,
as the following lemma shows.
\begin{lem}
$\boldsymbol{e}_{g}\approx\boldsymbol{e}_{f}$. \label{lem: e_g approx e_f}

The KL divergence can be re-expressed as:
\[
KL\left(\boldsymbol{e}_{g},\boldsymbol{e}_{f}\right)=\log\left[E\left(\exp\left[\xi\left(\boldsymbol{e}_{g}\right)-E\left(\xi\left(\boldsymbol{e}_{g}\right)\right)\right]\right)\right]
\]
 and then upper-bounded or approximated using $\text{var}\left[\xi\left(\boldsymbol{e}_{g}\right)\right]$:
\begin{align*}
KL\left(\boldsymbol{e}_{g},\boldsymbol{e}_{f}\right) & \leq\log\left[1+\frac{1}{2}\exp\left(\max_{\boldsymbol{e}}\left(\xi\left(\boldsymbol{e}\right)\right)-E\left(\xi\left(\boldsymbol{e}_{g}\right)\right)\right)\text{var}\left[\xi\left(\boldsymbol{e}_{g}\right)\right]\right]\\
KL\left(\boldsymbol{e}_{g},\boldsymbol{e}_{f}\right) & \approx\frac{1}{2}\text{var}\left[\xi\left(\boldsymbol{e}_{g}\right)\right]\\
 & \approx\frac{1}{2}\left[\frac{E\left(r_{g}^{3}\right)}{6}\right]^{2}\text{var}\left[\Delta_{3}\left(\boldsymbol{e}_{g}\right)\right]
\end{align*}
Another useful upper-bound and approximation is found by separating
the summands in $\xi\left(\boldsymbol{e}\right)$. Noting $\xi_{ELBO}\left(\boldsymbol{e}\right)=E_{g}\left(\frac{r_{g}^{2}}{2}-\varphi_{\boldsymbol{e}}\left(r_{g}\right)\right)$
the ELBO approximation of $\xi\left(\boldsymbol{e}\right)$, we have:
\begin{align*}
KL\left(\boldsymbol{e}_{g},\boldsymbol{e}_{f}\right) & \leq\frac{1}{2}\log\left[E\left[\exp\left(2\xi_{ELBO}\left(\boldsymbol{e}_{g}\right)-2E\left[\xi_{ELBO}\left(\boldsymbol{e}_{g}\right)\right]\right)\right]\right]+\frac{1}{2}\log\left[E\left[\exp\left(2\epsilon_{1}\left(\boldsymbol{e}_{g}\right)-2E\left[\epsilon_{1}\left(\boldsymbol{e}_{g}\right)\right]\right)\right]\right]\\
 & \lessapprox\frac{1}{2}\log\left[E\left[\exp\left(2\xi_{ELBO}\left(\boldsymbol{e}_{g}\right)-2E\left[\xi_{ELBO}\left(\boldsymbol{e}_{g}\right)\right]\right)\right]\right]+\text{var}\left(\epsilon_{1}\left(\boldsymbol{e}_{g}\right)\right)
\end{align*}
We can also use $\xi\left(\boldsymbol{e}\right)=\frac{\Delta_{3}\left(\boldsymbol{e}_{g}\right)E\left(r_{g}^{3}\right)}{6}+\epsilon_{1}\left(\boldsymbol{e}\right)+\epsilon_{2}\left(\boldsymbol{e}\right)$:
\begin{align*}
KL\left(\boldsymbol{e}_{g},\boldsymbol{e}_{f}\right) & \leq\frac{1}{2}\log\left[E\left[\exp\left(2\frac{\Delta_{3}\left(\boldsymbol{e}_{g}\right)E\left(r_{g}^{3}\right)}{6}\right)\right]\right]+\frac{1}{2}\log\left[E\left[\exp\left(2\epsilon_{1}\left(\boldsymbol{e}_{g}\right)+2\epsilon_{2}\left(\boldsymbol{e}_{g}\right)-2E\left[\epsilon_{1}\left(\boldsymbol{e}_{g}\right)+\epsilon_{2}\left(\boldsymbol{e}_{g}\right)\right]\right)\right]\right]\\
 & \lessapprox\left[\frac{E\left(r_{g}^{3}\right)}{6}\right]^{2}\text{var}\left[\Delta_{3}\left(\boldsymbol{e}_{g}\right)\right]+\text{var}\left[\epsilon_{1}\left(\boldsymbol{e}_{g}\right)+\epsilon_{2}\left(\boldsymbol{e}_{g}\right)\right]
\end{align*}
\end{lem}
\begin{proof}
Let us start by rewriting the KL divergence to make $\tilde{f}$ and
then $\xi=\log\left[\tilde{f}\right]$ appear:
\begin{align*}
KL\left(\boldsymbol{e}_{g},\boldsymbol{e}_{f}\right) & =E\left(\log\frac{g\left(\boldsymbol{e}_{g}\right)}{f\left(\boldsymbol{e}_{g}\right)}\right)\\
 & =E\left(\log\frac{g\left(\boldsymbol{e}_{g}\right)\left(\int\tilde{f}\right)}{\tilde{f}\left(\boldsymbol{e}_{g}\right)}\right)\\
 & =E\left(\log\frac{g\left(\boldsymbol{e}_{g}\right)}{\tilde{f}\left(\boldsymbol{e}_{g}\right)}\right)+\log\left[\int\tilde{f}\right]\\
 & =E\left(\log\frac{g\left(\boldsymbol{e}_{g}\right)}{\tilde{f}\left(\boldsymbol{e}_{g}\right)}\right)+\log\left[E\left(\frac{\tilde{f}\left(\boldsymbol{e}_{g}\right)}{g\left(\boldsymbol{e}_{g}\right)}\right)\right]\\
 & =\log\left[E\left(\frac{\tilde{f}\left(\boldsymbol{e}_{g}\right)}{g\left(\boldsymbol{e}_{g}\right)}\right)\right]-E\left(\log\frac{\tilde{f}\left(\boldsymbol{e}_{g}\right)}{g\left(\boldsymbol{e}_{g}\right)}\right)\\
 & =\log\left[E\left(\exp\left[\log\frac{\tilde{f}\left(\boldsymbol{e}_{g}\right)}{g\left(\boldsymbol{e}_{g}\right)}\right]\right)\right]-E\left(\log\frac{\tilde{f}\left(\boldsymbol{e}_{g}\right)}{g\left(\boldsymbol{e}_{g}\right)}\right)
\end{align*}
In this final expression, we can further remove $g\left(\boldsymbol{e}_{g}\right)$
which is constant, and bring $E\left(\log\tilde{f}\right)$ into the
first term:
\begin{equation}
KL\left(\boldsymbol{e}_{g},\boldsymbol{e}_{f}\right)=\log\left[E\left(\exp\left[\xi\left(\boldsymbol{e}_{g}\right)-E\left(\xi\left(\boldsymbol{e}_{g}\right)\right)\right]\right)\right]
\end{equation}
we are thus left with a final expression for the KL divergence which
only involves expected values of $\xi\left(\boldsymbol{e}_{g}\right)=\log\left[\tilde{f}\left(\boldsymbol{e}_{g}\right)\right]$.

We now need to upper-bound this KL divergence to prove Lemma \ref{lem: e_g approx e_f}.
We do so by upper-bounding the exponential function. We will use the
following bound which holds for all $x\leq M$:
\[
\exp\left(x\right)\leq1+x+\frac{x^{2}}{2}\exp\left(M\right)
\]
This bound holds because these two functions have the same value and
the same first derivative at $0$, and their second derivatives verify:
\[
\exp\left(x\right)\leq\exp\left(M\right)
\]
Thus, their difference: $1+x+\frac{x^{2}}{2}\exp\left(M\right)-\exp\left(x\right)$,
is a convex function with minimum at 0, from which we deduce the inequality.

Armed with this upper-bound on the exponential function, let us return
to the KL divergence, denoting $\xi\left(\boldsymbol{e}\right)=\log\left[\tilde{f}\left(\boldsymbol{e}\right)\right]$:
\begin{align*}
KL\left(\boldsymbol{e}_{g},\boldsymbol{e}_{f}\right) & =\log\left[E\left(\exp\left[-\xi\left(\boldsymbol{e}_{g}\right)+E\left(\xi\left(\boldsymbol{e}_{g}\right)\right)\right]\right)\right]\\
 & \leq\log\left[1-E\left(\xi\left(\boldsymbol{e}_{g}\right)+E\left(\xi\left(\boldsymbol{e}_{g}\right)\right)\right)+\exp\left(\max_{\boldsymbol{e}}\left(\xi\left(\boldsymbol{e}\right)\right)-E\left(\xi\left(\boldsymbol{e}_{g}\right)\right)\right)\frac{E\left(\left[\xi\left(\boldsymbol{e}_{g}\right)-E\left(\xi\left(\boldsymbol{e}_{g}\right)\right)\right]^{2}\right)}{2}\right]\\
 & \leq\log\left[1+0+\exp\left(\max_{\boldsymbol{e}}\left(\xi\left(\boldsymbol{e}\right)\right)-E\left(\xi\left(\boldsymbol{e}_{g}\right)\right)\right)\frac{E\left(\left[\xi\left(\boldsymbol{e}_{g}\right)-E\left(\xi\left(\boldsymbol{e}_{g}\right)\right)\right]^{2}\right)}{2}\right]\\
 & \leq\log\left[1+\frac{1}{2}\exp\left(\max_{\boldsymbol{e}}\left(\xi\left(\boldsymbol{e}\right)\right)-E\left(\xi\left(\boldsymbol{e}_{g}\right)\right)\right)\text{var}\left[\xi\left(\boldsymbol{e}_{g}\right)\right]\right]
\end{align*}

We could also use a simple approximation of $\exp\left(x\right)$
and $\log\left(1+x\right)$ around 0. We would then get:
\begin{align*}
\exp\left(x\right) & \approx1+x+\frac{x^{2}}{2}\\
\log\left(1+y\right) & \approx y\\
KL\left(\boldsymbol{e}_{g},\boldsymbol{e}_{f}\right) & \approx\log\left[1+\frac{1}{2}\text{var}\left[\xi\left(\boldsymbol{e}_{g}\right)\right]\right]\\
 & \approx\frac{1}{2}\text{var}\left[\xi\left(\boldsymbol{e}_{g}\right)\right]
\end{align*}
Now we will work on splitting $\xi\left(\boldsymbol{e}\right)$. From
Lemma \ref{lem:Density-of-. z_g e_f}, we know that:
\begin{align*}
\xi\left(\boldsymbol{e}\right) & =\log\left[\tilde{f}\left(\boldsymbol{e}\right)\right]=\log\left[\int\tilde{f}\left(r|\boldsymbol{e}\right)dr\right]\\
 & =C+\log\left[\int r^{d-1}\exp\left(-\varphi_{\boldsymbol{e}}\left(r\right)\right)dr\right]
\end{align*}
where the constant $C$ doesn't depend on $\boldsymbol{e}$.

A good approximation of integrals of these form is the Evidence Lower
Bound (ELBO; \citet{murphy2012machine} Chapter 21):
\[
\xi\left(\boldsymbol{e}\right)\approx E_{g}\left(\frac{r_{g}^{2}}{2}-\varphi_{\boldsymbol{e}}\left(r\right)\right)
\]
Like its name indicates, the ELBO lower-bounds $\xi\left(\boldsymbol{e}\right)$:
the difference between the two is precisely equal to $KL\left(r_{g},r_{f}|\boldsymbol{e}\right)$.

We have bounded the KL divergence between $r_{g}$ and $r_{f}|\boldsymbol{e}$.
Thus, we can precisely control the error using the KL divergence upper-bound
of lemma \ref{lem: z_g approx z_f | e}. Thus :
\begin{align*}
\xi\left(\boldsymbol{e}\right) & =C_{2}+E\left(\frac{r_{g}^{2}}{2}-\varphi_{\boldsymbol{e}}\left(r_{g}\right)\right)+\epsilon_{1}\left(\boldsymbol{e}\right)\\
\epsilon_{1}\left(\boldsymbol{e}\right) & =KL\left(r_{g},r_{f}|\boldsymbol{e}\right)\\
0\leq\epsilon_{1}\left(\boldsymbol{e}\right) & \leq\frac{E\left[4r_{g}\left(\varphi_{\boldsymbol{e}}^{'}\left(r_{g}\right)-r_{g}\right)^{2}\right]}{{\displaystyle \min_{z\geq0}\left[\psi_{f,\boldsymbol{e}}^{''}\left(z\right)\right]}}
\end{align*}
Note that the constant $C_{2}$ doesn't depend on $\boldsymbol{e}$
and will vanish when in the KL divergence since we are interested
in $\xi\left(\boldsymbol{e}_{g}\right)-E\left(\xi\left(\boldsymbol{e}_{g}\right)\right)$.

The term that causes the majority of the oscillations of $\xi\left(\boldsymbol{e}\right)$
is $E\left(\frac{r_{g}^{2}}{2}-\varphi_{\boldsymbol{e}}\left(r_{g}\right)\right)$.
Let us now compute the size of this term:
\begin{align*}
\left|\frac{r_{g}^{2}}{2}-\varphi_{\boldsymbol{e}}\left(r_{g}\right)+\Delta_{3}\left(\boldsymbol{e}\right)\frac{r_{g}^{3}}{3!}\right| & \leq\Delta_{4}\left(\boldsymbol{e}\right)\frac{r_{g}^{4}}{4!}\\
\left|E\left(\frac{r_{g}^{2}}{2}-\varphi_{\boldsymbol{e}}\left(r_{g}\right)\right)-\frac{\Delta_{3}\left(\boldsymbol{e}\right)}{3!}E\left(r_{g}^{3}\right)\right| & \leq\frac{\Delta_{4}\left(\boldsymbol{e}\right)}{4!}E\left(r_{g}^{4}\right)
\end{align*}

We can then rewrite $\xi\left(\boldsymbol{e}\right)$ as:
\begin{align*}
\xi\left(\boldsymbol{e}\right) & =C_{2}-\frac{\Delta_{3}\left(\boldsymbol{e}\right)E\left(r_{g}^{3}\right)}{6}+\epsilon_{1}\left(\boldsymbol{e}\right)+\epsilon_{2}\left(\boldsymbol{e}\right)\\
0\leq\epsilon_{1}\left(\boldsymbol{e}\right) & \leq\frac{E\left[4r_{g}\left(\varphi_{\boldsymbol{e}}^{'}\left(r_{g}\right)-r_{g}\right)^{2}\right]}{{\displaystyle \min_{z\geq0}\left[\psi_{f,\boldsymbol{e}}^{''}\left(z\right)\right]}}\\
\left|\epsilon_{2}\left(\boldsymbol{e}\right)\right| & \leq\frac{\Delta_{4}\left(\boldsymbol{e}\right)}{4!}E\left(r_{g}^{4}\right)
\end{align*}

Finally, we can use this decomposition of $\xi\left(\boldsymbol{e}\right)$
into multiple terms when considering the KL divergence. To prove this,
consider the following function:
\[
K\left(\lambda\right)=\log\left[E\left(\exp\left[\lambda A\left(\boldsymbol{e}_{g}\right)+\left(1-\lambda\right)B\left(\boldsymbol{e}_{g}\right)\right]\right)\right]
\]
where $A\left(\boldsymbol{e}_{g}\right)$ and $B\left(\boldsymbol{e}_{g}\right)$
are arbitrary functions. The first derivative of $K$ is:
\begin{align*}
K^{'}\left(\lambda\right) & =\frac{\int g\left(\boldsymbol{e}\right)d\boldsymbol{e}\left[A\left(\boldsymbol{e}\right)-B\left(\boldsymbol{e}\right)\right]\exp\left[\lambda A\left(\boldsymbol{e}\right)+\left(1-\lambda\right)B\left(\boldsymbol{e}\right)\right]}{\int g\left(\boldsymbol{e}\right)d\boldsymbol{e}\exp\left[\lambda A\left(\boldsymbol{e}\right)+\left(1-\lambda\right)B\left(\boldsymbol{e}\right)\right]}\\
 & =E\left[A\left(\boldsymbol{e}\right)-B\left(\boldsymbol{e}\right)|\lambda\right]
\end{align*}
where the expected value is computed against the normalized density
\[
g\left(\boldsymbol{e}|\lambda\right)=\frac{g\left(\boldsymbol{e}\right)\exp\left[\lambda A\left(\boldsymbol{e}\right)+\left(1-\lambda\right)B\left(\boldsymbol{e}\right)\right]}{\int g\left(\boldsymbol{e}\right)d\boldsymbol{e}\exp\left[\lambda A\left(\boldsymbol{e}\right)+\left(1-\lambda\right)B\left(\boldsymbol{e}\right)\right]}
\]
The second derivative of $K$ is:
\begin{align*}
K^{''}\left(\lambda\right) & =E\left[\left(A\left(\boldsymbol{e}\right)-B\left(\boldsymbol{e}\right)\right)^{2}|\lambda\right]-E\left[A\left(\boldsymbol{e}\right)-B\left(\boldsymbol{e}\right)|\lambda\right]^{2}\\
 & =\text{var}\left[A\left(\boldsymbol{e}\right)-B\left(\boldsymbol{e}\right)|\lambda\right]
\end{align*}
Critically: $K^{''}\left(\lambda\right)\geq0$ and $K$ is thus a
convex function.

By the convexity of $K\left(\lambda\right)$, we have:
\[
K\left(0.5\right)\leq K\left(0\right)+K\left(1\right)
\]
Now, take $A\left(\boldsymbol{e}\right)=2\frac{\Delta_{3}\left(\boldsymbol{e}_{g}\right)E\left(r_{g}^{3}\right)}{6}-E\left(2\frac{\Delta_{3}\left(\boldsymbol{e}_{g}\right)E\left(r_{g}^{3}\right)}{6}\right)$
and $B\left(\boldsymbol{e}\right)=2\epsilon_{1}\left(\boldsymbol{e}\right)+2\epsilon_{2}\left(\boldsymbol{e}\right)-2E\left[\epsilon_{1}\left(\boldsymbol{e}\right)+\epsilon_{2}\left(\boldsymbol{e}\right)\right]$.
By combining the symmetry of $\boldsymbol{e}_{g}$ with the asymmetry
of $\Delta_{3}\left(\boldsymbol{e}_{g}\right)$, we have that $E\left(2\frac{\Delta_{3}\left(\boldsymbol{e}_{g}\right)E\left(r_{g}^{3}\right)}{6}\right)=0$.

The values of $K$ are then:
\begin{align*}
K\left(0.5\right) & =\log\left[E\left[\exp\left(\xi\left(\boldsymbol{e}\right)-E\left(\xi\left(\boldsymbol{e}\right)\right)\right)\right]\right]\\
K\left(0\right) & =\log\left[E\left[\exp\left(2\frac{\Delta_{3}\left(\boldsymbol{e}_{g}\right)E\left(r_{g}^{3}\right)}{6}\right)\right]\right]\\
K\left(1\right) & =\log\left[E\left[\exp\left(2\epsilon\left(\boldsymbol{e}\right)+2\epsilon^{'}\left(\boldsymbol{e}\right)-2E\left[\epsilon\left(\boldsymbol{e}\right)+\epsilon^{'}\left(\boldsymbol{e}\right)\right]\right)\right]\right]
\end{align*}
and we find the final inequality:
\[
\frac{1}{2}\log\left[E\left[\exp\left(2\frac{\Delta_{3}\left(\boldsymbol{e}_{g}\right)E\left(r_{g}^{3}\right)}{6}\right)\right]\right]+\frac{1}{2}\log\left[E\left[\exp\left(2\epsilon\left(\boldsymbol{e}_{g}\right)+2\epsilon^{'}\left(\boldsymbol{e}_{g}\right)-2E\left[\epsilon\left(\boldsymbol{e}_{g}\right)+\epsilon^{'}\left(\boldsymbol{e}_{g}\right)\right]\right)\right]\right]
\]
Note that for any $\alpha\in]0,1[$, we could write: 
\[
\xi\left(\boldsymbol{e}\right)-E\left(\xi\left(\boldsymbol{e}\right)\right)=\alpha\frac{1}{\alpha}\frac{\Delta_{3}\left(\boldsymbol{e}_{g}\right)E\left(r_{g}^{3}\right)}{6}+\left(1-\alpha\right)\frac{1}{1-\alpha}\left(\epsilon\left(\boldsymbol{e}_{g}\right)+\epsilon^{'}\left(\boldsymbol{e}_{g}\right)-E\left[\epsilon\left(\boldsymbol{e}_{g}\right)+\epsilon^{'}\left(\boldsymbol{e}_{g}\right)\right]\right)
\]
which can give other useful upper-bounds on the KL divergence.
\end{proof}

\subsection{Approximating $f$}

We are now finally ready to combine all of the preceding lemmas to
state the full form of our theorem:
\begin{enumerate}
\item We have determined the density of the pairs $\left(z_{g},\boldsymbol{e}_{g}\right)$
and $\left(z_{f},\boldsymbol{e}_{f}\right)$.
\item We have computed the KL divergence between $z_{g}$ and the conditional
distribution $z_{f}|\boldsymbol{e}$.
\item We have computed the KL divergence between $\boldsymbol{e}_{g}$ and
$\boldsymbol{e}_{f}$.
\end{enumerate}
Now, the only step that remains consists in combining those results,
which we are able to do because the KL divergence is invariant to
changes of variables.
\begin{thm}
A detailed upper-bound on the KL divergence.\label{thm: Detailed theorem}

The KL divergence between $g$ and $f$ can be decomposed as:
\[
KL\left(g,f\right)=KL\left(\boldsymbol{e}_{f},\boldsymbol{e}_{g}\right)+E\left(KL\left(z_{g},z_{f}|\boldsymbol{e}_{g}\right)\right)
\]
where the second expected value computes the mean of the function
$\boldsymbol{e}\rightarrow KL\left(z_{g},z_{f}|\boldsymbol{e}\right)$
over the random variable $\boldsymbol{e}_{g}$.

The KL divergence can be upper-bounded as:
\begin{align}
KL\left(g,f\right) & \leq\frac{1}{2}\log\left[E\left[\exp\left(2\xi_{ELBO}\left(\boldsymbol{e}_{g}\right)-2E\left[\xi_{ELBO}\left(\boldsymbol{e}_{g}\right)\right]\right)\right]\right]+\frac{1}{2}\log\left[E\left[\exp\left(2\epsilon_{1}\left(\boldsymbol{e}_{g}\right)-2E\left[\epsilon_{1}\left(\boldsymbol{e}_{g}\right)\right]\right)\right]\right]\nonumber \\
 & \ \ \ \ \ \ \ \ +E_{\boldsymbol{e}_{g}}\left[\frac{E_{r_{g}}\left[4r_{g}\left(\varphi_{\boldsymbol{e}_{g}}^{'}\left(r_{g}\right)-r_{g}\right)^{2}\right]}{{\displaystyle \min_{z\geq0}\left[\psi_{f,\boldsymbol{e}_{g}}^{''}\left(z\right)\right]}}\right]
\end{align}
A computable approximation is found by approximating the term involving
$\epsilon_{1}\left(\boldsymbol{e}_{g}\right)$:
\begin{align}
KL\left(g,f\right) & \leq\frac{1}{2}\log\left[E\left[\exp\left(2\xi_{ELBO}\left(\boldsymbol{e}_{g}\right)-2E\left[\xi_{ELBO}\left(\boldsymbol{e}_{g}\right)\right]\right)\right]\right]+E\left[\left(\frac{E_{r_{g}}\left[4r_{g}\left(\varphi_{\boldsymbol{e}_{g}}^{'}\left(r_{g}\right)-r_{g}\right)^{2}\right]}{{\displaystyle \min_{z\geq0}\left[\psi_{f,\boldsymbol{e}_{g}}^{''}\left(z\right)\right]}}\right)^{2}\right]\nonumber \\
 & \ \ \ \ \ \ \ \ +E_{\boldsymbol{e}_{g}}\left[\frac{E_{r_{g}}\left[4r_{g}\left(\varphi_{\boldsymbol{e}_{g}}^{'}\left(r_{g}\right)-r_{g}\right)^{2}\right]}{{\displaystyle \min_{z\geq0}\left[\psi_{f,\boldsymbol{e}_{g}}^{''}\left(z\right)\right]}}\right]
\end{align}
in which all expected values need to be approximated by sampling from
$\boldsymbol{e}_{g}$.

We can also approximate this bound by keeping only the terms that
depend on $\Delta_{3}\left(\boldsymbol{e}\right)$:
\begin{equation}
KL\left(g,f\right)\lessapprox E\left[\Delta_{3}\left(\boldsymbol{e}_{g}\right)^{2}\right]\left(\frac{2}{\sqrt{3}\sqrt{2d-1}}\frac{\Gamma\left[\left(d+5\right)/2\right]}{\Gamma\left[d/2\right]}+\frac{1}{9}\left[\frac{\Gamma\left[\left(d+3\right)/2\right]}{\Gamma\left[d/2\right]}\right]^{2}\right)
\end{equation}
\end{thm}
\begin{proof}
Let us prove the formula for the decomposition of the KL divergence.
We start from the normal formula for the KL divergence:
\[
KL\left(g,f\right)=\int g\left(\boldsymbol{\theta}\right)\log\left[\frac{g\left(\boldsymbol{\theta}\right)}{f\left(\boldsymbol{\theta}\right)}\right]d\boldsymbol{\theta}
\]
We then perform the change of variable: $\boldsymbol{\theta}\rightarrow\left(z,\boldsymbol{e}\right)$
and get:
\begin{align*}
KL\left(g,f\right) & =\int g\left(z\right)g\left(\boldsymbol{e}\right)\log\left[\frac{g\left(z\right)g\left(\boldsymbol{e}\right)}{f\left(z,\boldsymbol{e}\right)}\right]dzd\boldsymbol{e}\\
 & =\int g\left(z\right)g\left(\boldsymbol{e}\right)\log\left[\frac{g\left(z\right)g\left(\boldsymbol{e}\right)}{f\left(z|\boldsymbol{e}\right)f\left(\boldsymbol{e}\right)}\right]dzd\boldsymbol{e}\\
 & =\int g\left(z\right)g\left(\boldsymbol{e}\right)\left(\log\left[\frac{g\left(\boldsymbol{e}\right)}{f\left(\boldsymbol{e}\right)}\right]+\log\left[\frac{g\left(z\right)}{f\left(z|\boldsymbol{e}\right)}\right]\right)dzd\boldsymbol{e}\\
 & =\left(\int g\left(\boldsymbol{e}\right)\log\left[\frac{g\left(\boldsymbol{e}\right)}{f\left(\boldsymbol{e}\right)}\right]\right)+\int g\left(\boldsymbol{e}\right)\left(\int g\left(z\right)\log\left[\frac{g\left(z\right)}{f\left(z|\boldsymbol{e}\right)}\right]dz\right)d\boldsymbol{e}\\
 & =KL\left(\boldsymbol{e}_{g},\boldsymbol{e}_{f}\right)+\int g\left(\boldsymbol{e}\right)\left(KL\left(z_{g},z_{f}|\boldsymbol{e}\right)\right)d\boldsymbol{e}\\
 & =KL\left(\boldsymbol{e}_{g},\boldsymbol{e}_{f}\right)+E\left[KL\left(z_{g},z_{f}|\boldsymbol{e}_{g}\right)\right]
\end{align*}

We then combine this expression for the KL divergence with Lemmas
\ref{lem: z_g approx z_f | e} and \ref{lem: e_g approx e_f}. Using
the most general approximation form from these theorems, we get:
\begin{align*}
KL\left(g,f\right) & \leq\frac{1}{2}\log\left[E\left[\exp\left(2\xi_{ELBO}\left(\boldsymbol{e}_{g}\right)-2E\left[\xi_{ELBO}\left(\boldsymbol{e}_{g}\right)\right]\right)\right]\right]+\frac{1}{2}\log\left[E\left[\exp\left(2\epsilon_{1}\left(\boldsymbol{e}_{g}\right)-2E\left[\epsilon_{1}\left(\boldsymbol{e}_{g}\right)\right]\right)\right]\right]\\
 & \ \ \ \ \ \ \ \ +E_{\boldsymbol{e}_{g}}\left[\frac{E_{r_{g}}\left[4r_{g}\left(\varphi_{\boldsymbol{e}_{g}}^{'}\left(r_{g}\right)-r_{g}\right)^{2}\right]}{{\displaystyle \min_{z\geq0}\left[\psi_{f,\boldsymbol{e}_{g}}^{''}\left(z\right)\right]}}\right]
\end{align*}
in which $\epsilon_{1}\left(\boldsymbol{e}\right)=KL\left(z_{g},z_{f}|\boldsymbol{e}\right)$
is bounded:
\[
0\leq\epsilon_{1}\left(\boldsymbol{e}\right)\leq\frac{E_{r_{g}}\left[4r_{g}\left(\varphi_{\boldsymbol{e}_{g}}^{'}\left(r_{g}\right)-r_{g}\right)^{2}\right]}{{\displaystyle \min_{z\geq0}\left[\psi_{f,\boldsymbol{e}_{g}}^{''}\left(z\right)\right]}}
\]
The term $\frac{1}{2}\log\left[E\left[\exp\left(2\epsilon_{1}\left(\boldsymbol{e}_{g}\right)-2E\left[\epsilon_{1}\left(\boldsymbol{e}_{g}\right)\right]\right)\right]\right]$
is problematic since we only have an upper-bound on each term $\epsilon_{1}\left(\boldsymbol{e}\right)$.
One solution consists in using the approximation (see the proof of
Lemma \ref{lem: e_g approx e_f} in which we have detailed further
how to approximate this term):
\begin{align*}
\frac{1}{2}\log\left[E\left[\exp\left(2\epsilon\left(\boldsymbol{e}_{g}\right)-2E\left[\epsilon\left(\boldsymbol{e}_{g}\right)\right]\right)\right]\right] & \approx\text{var}\left[\epsilon_{1}\left(\boldsymbol{e}_{g}\right)\right]\\
 & \lessapprox E\left[\left(\epsilon_{1}\left(\boldsymbol{e}_{g}\right)\right)^{2}\right]\\
 & \lessapprox E\left[\left(\frac{E_{r_{g}}\left[4r_{g}\left(\varphi_{\boldsymbol{e}_{g}}^{'}\left(r_{g}\right)-r_{g}\right)^{2}\right]}{{\displaystyle \min_{z\geq0}\left[\psi_{f,\boldsymbol{e}_{g}}^{''}\left(z\right)\right]}}\right)^{2}\right]
\end{align*}
The final formula of the theorem is simply found by combining the
approximate forms of Lemmas \ref{lem: z_g approx z_f | e} and \ref{lem: e_g approx e_f},
yielding:
\[
KL\left(g,f\right)\lessapprox\frac{1}{2}\left[\frac{E\left(r_{g}^{3}\right)}{6}\right]^{2}E\left[\Delta_{3}\left(\boldsymbol{e}_{g}\right)^{2}\right]+\frac{E\left(r_{g}^{5}\right)}{2\sqrt{6}\sqrt{2d-1}}E\left[\Delta_{3}\left(\boldsymbol{e}_{g}\right)^{2}\right]
\]
We then use the expression for the moments of $r_{g}$:
\begin{align*}
E\left(r_{g}^{3}\right) & =2\sqrt{2}\frac{\Gamma\left[\left(d+3\right)/2\right]}{\Gamma\left[d/2\right]}\\
E\left(r_{g}^{5}\right) & =4\sqrt{2}\frac{\Gamma\left[\left(d+5\right)/2\right]}{\Gamma\left[d/2\right]}
\end{align*}
We finally obtain:
\begin{align*}
KL\left(g,f\right) & \lessapprox\frac{1}{2}\left[\frac{2\sqrt{2}\frac{\Gamma\left[\left(d+3\right)/2\right]}{\Gamma\left[d/2\right]}}{6}\right]^{2}E\left[\Delta_{3}\left(\boldsymbol{e}_{g}\right)^{2}\right]+\frac{4\sqrt{2}\frac{\Gamma\left[\left(d+5\right)/2\right]}{\Gamma\left[d/2\right]}}{2\sqrt{6}\sqrt{2d-1}}E\left[\Delta_{3}\left(\boldsymbol{e}_{g}\right)^{2}\right]\\
 & \lessapprox\frac{1}{2}\left[\frac{\sqrt{2}\Gamma\left[\left(d+3\right)/2\right]}{3\Gamma\left[d/2\right]}\right]^{2}E\left[\Delta_{3}\left(\boldsymbol{e}_{g}\right)^{2}\right]+\frac{2}{\sqrt{3}\sqrt{2d-1}}\frac{\Gamma\left[\left(d+5\right)/2\right]}{\Gamma\left[d/2\right]}E\left[\Delta_{3}\left(\boldsymbol{e}_{g}\right)^{2}\right]\\
 & \lessapprox E\left[\Delta_{3}\left(\boldsymbol{e}_{g}\right)^{2}\right]\left(\frac{1}{2}\frac{2}{9}\left[\frac{\Gamma\left[\left(d+3\right)/2\right]}{\Gamma\left[d/2\right]}\right]^{2}+\frac{2}{\sqrt{3}\sqrt{2d-1}}\frac{\Gamma\left[\left(d+5\right)/2\right]}{\Gamma\left[d/2\right]}\right)\\
 & \lessapprox E\left[\Delta_{3}\left(\boldsymbol{e}_{g}\right)^{2}\right]\left(\frac{2}{\sqrt{3}\sqrt{2d-1}}\frac{\Gamma\left[\left(d+5\right)/2\right]}{\Gamma\left[d/2\right]}+\frac{1}{9}\left[\frac{\Gamma\left[\left(d+3\right)/2\right]}{\Gamma\left[d/2\right]}\right]^{2}\right)
\end{align*}
\end{proof}

\section{Details of the logistic regression example\label{sec:Details-of-the logistic regression}}

In order to assess whether the rough approximation of the bound could
prove useful, we have tested it in a simple example: logistic regression.
In this model, the data $\mathcal{D}$ is composed of the $n$ pairs:
$\left(y_{i},\boldsymbol{x}_{i}\right)\in\left\{ -1,1\right\} *\mathbb{R}^{d}$
corresponding to class labels $y_{i}$ and predictors (or covariates)
$\boldsymbol{x}_{i}$. If we choose the prior to be Gaussian (with
mean 0 and covariance matrix $\sigma_{0}^{2}I_{d}$), the posterior
is then:
\begin{equation}
f\left(\boldsymbol{\theta}\right)=\frac{1}{Z}\exp\left(-\frac{1}{2}\sum_{j=1}^{d}\frac{\left(\theta_{i}\right)^{2}}{\sigma_{0}^{2}}\right)\prod_{i=1}^{n}\frac{1}{1+\exp\left[-y_{i}\left(\boldsymbol{\theta}.\boldsymbol{x}_{i}\right)\right]}
\end{equation}
which is log-concave.

In order to test our theorem, we needed to:
\begin{itemize}
\item Compute the Laplace approximation of the posterior, i.e: compute $\boldsymbol{\theta}^{\star}$
and $H\phi_{f}\left(\boldsymbol{\theta}^{\star}\right)$.
\item Compute the third derivative tensor: $\phi_{f}^{\left(3\right)}\left(\boldsymbol{\theta}^{\star}\right)$
and use it to deduce $\Delta_{3}\left(\boldsymbol{e}\right)$:
\[
\Delta_{3}\left(\boldsymbol{e}\right)=\phi_{f}^{\left(3\right)}\left(\boldsymbol{\theta}^{\star}\right)\left[\Sigma^{1/2}\boldsymbol{e},\Sigma^{1/2}\boldsymbol{e},\Sigma^{1/2}\boldsymbol{e}\right]
\]
\item Approximate the real value of the KL divergence 
\[
KL\left(g,f\right)=\int g\left(\boldsymbol{\theta}\right)\log\left[\frac{g\left(\boldsymbol{\theta}\right)Z}{\tilde{f}\left(\boldsymbol{\theta}\right)}\right]d\boldsymbol{\theta}
\]
which requires approximating the normalizing constant $Z$.
\end{itemize}
We computed these values in the following way:
\begin{itemize}
\item The MAP value $\boldsymbol{\theta}^{\star}$ was computed using a
line-search gradient descent.
\item The normalizing constant was approximated using a Markov Chain Monte
Carlo sampling method:
\begin{itemize}
\item Using standard Metropolis-Hastings, we generated $k$ samples from
$f\left(\boldsymbol{\theta}\right)$: $\boldsymbol{\theta}_{i}$.\\
In the examples reported, we generated $10^{7}$ samples and we then
kept one in $1000$ thus giving $k=10^{4}$.
\item Using the samples, we approximated:
\[
\frac{1}{Z}=E\left(\frac{g\left(\boldsymbol{\theta}_{f}\right)}{\tilde{f}\left(\boldsymbol{\theta}_{f}\right)}\right)\approx\frac{1}{k}\sum_{i}\frac{g\left(\boldsymbol{\theta}_{i}\right)}{\tilde{f}\left(\boldsymbol{\theta}_{i}\right)}
\]
\end{itemize}
\item The KL divergence was also computed by sampling:
\begin{itemize}
\item We generated $k_{2}$ samples from the Gaussian density $g\left(\boldsymbol{\theta}\right)$.\\
In the examples reported, $k_{2}=10^{5}$.
\item We used those samples to approximate the KL divergence:
\[
KL\left(g,f\right)=E\left(\log\left[\frac{g\left(\boldsymbol{\theta}_{g}\right)Z}{\tilde{f}\left(\boldsymbol{\theta}_{g}\right)}\right]\right)
\]
in which we reused the approximation of $Z$.
\end{itemize}
\end{itemize}
The data was generated from the logistic regression model itself:
\begin{itemize}
\item First, we picked the covariates $\boldsymbol{x}_{i}$ according to
a Gaussian distribution with mean 0 and variance the identity matrix.
\item Then, we picked the true value of the parameter $\boldsymbol{\theta}_{0}$
from a Gaussian distribution with mean 0 and variance $d^{-1/2}I_{d}$.
This scaling ensures that the values $\boldsymbol{\theta}.\boldsymbol{x}_{i}$
are of order $1$ no matter the dimension $d$.
\item Finally, the label $y_{i}$ was picked according to the logistic density:
\[
\mathbb{P}\left(y_{i}=\pm1\right)=\frac{1}{1+\exp\left[-y_{i}\left(\boldsymbol{\theta}.\boldsymbol{x}_{i}\right)\right]}
\]
\end{itemize}

\end{document}